\documentclass[submission,copyright,creativecommons]{eptcs}
\providecommand{\event}{ICLP 2023}

\usepackage{iftex}

\ifpdf
  \usepackage{underscore}         
  \usepackage[T1]{fontenc}        
\else
  \usepackage{breakurl}           
\fi

\usepackage{times}
\usepackage{helvet}
\usepackage{courier}
\usepackage{hyperref}
\usepackage{amsfonts}
\usepackage{amsthm}
\usepackage{amsmath}
\usepackage{graphicx}

\newtheorem{thm}{Theorem}

\newtheorem{cor}[thm]{Corollary}
\newtheorem{lem}[thm]{Lemma}

\newtheorem{conv}[thm]{Convention}
\newtheorem{prop}[thm]{Proposition}

\theoremstyle{remark}
\newtheorem{remark}[thm]{Remark}

\theoremstyle{definition}
\newtheorem{defn}[thm]{Definition}
\DeclareMathOperator{\sm}{softmax}

\title{On the Potential of CLIP for Compositional Logical Reasoning}

\author{Justin Brody
  \institute{Franklin and Marshall College\\ Lancaster, PA, USA}
  \email{justin.brody@fandm.edu}
}

\newcommand{\titlerunning}{Compositional Reasoning in CLIP}
\newcommand{\authorrunning}{J. Brody}
\hypersetup{
  bookmarksnumbered,
  pdftitle    = {\titlerunning},
  pdfauthor   = {\authorrunning},
  pdfsubject  = {EPTCS},               
}

\begin{document}
\maketitle
\begin{abstract}
  In this paper we explore the possibility of using OpenAI's CLIP to
  perform logically coherent grounded visual reasoning.  To that end,
  we formalize our terms and give a geometric analysis of how
  embeddings in CLIP's latent space would need to be configured in
  order for the system to be logically coherent.  Our main conclusion
  is that, as usually configured, CLIP cannot perform such reasoning.
\end{abstract}

\def\i{\mathbb{I}}
\def\c{\mathbb{C}}
\def\D{\mathcal{D}}
\def\A{\mathcal{A}}
\def\R{\mathbf{R}}
\def\cmodels{\models_{\scriptscriptstyle{C}}}
\def\cproves{\vdash_{\scriptscriptstyle{C}}}
\def\cneg{\neg_{\scriptscriptstyle{C}\,}}
\def\clor{\lor_{\scriptscriptstyle{C}}}
\def\cland{\land_{\scriptscriptstyle{C}}}

\def\L{\mathbf{L}}

\def\I{\mathcal{I}}

\section{Introduction}
Recent work in machine learning has seen stunning results in combining
generative language models with vision models.  By tying vision and
language, such models are capable of generating images from
descriptions \cite{ramesh2022hierarchical},
\cite{rombach2022highresolution} or engaging in sophisticated visual
question answering \cite{antol2015vqa} and other forms of {\it visual
  reasoning}.  One striking feature of models trained on the latter
class is the reasoning they do tends to produce new inferences via the
generative mechanisms of the language model rather than through any
kind of formal logical reasoning.  However, logical reasoning exhibits
a number of desirable properties, including the known soundness and
completeness of first-order logic.  These in turn constitute a
particular form of systemic {\it compositionality} -- the capacity
noted in symoblic engines to systematically combine symbols to produce
more complex representations \cite{sep-compositionality}.  Indeed, in
old debates between symbolic and connectionist approaches to
artificial intelligence, a frequent critique made by proponents of
the symbolic approach was that connectionist systems lacked such
compositionality.  To a certain degree, the tremendous success of
large language models and their derivatives has muted this criticism
-- indeed large language models are connectionist systems which easily
recombine symbolic tokens to produce complex sentences (see,
e.g. \cite{vaswani2017attention}).  Further, their obvious power leads
one to wonder whether the systematic compositionality found in
classical AI system might be present in contemporary neural systems as
well.

This in turn leads specifically to the question of whether a more
logic-based kind of reasoning can occur in vision-language models.
This paper presents an analytical study of this possibility in one
such model, OpenAI's CLIP \cite{pmlr-v139-radford21a}, which embeds
sets of images and sets of captions in a common vector space and uses
cosine similarity to measure how well each caption describes each
image.  We are specifically interested in the possiblity of using CLIP
as a basis for grounded, logical reasoning.  By the first, I simply
mean that reasoning about an image should in some way be connected to
the image ({\it grounded} in the image).  
The joint-embedding mechanism of CLIP, along with cosine similiarity,
provides precisely such a mechanism.  Logical reasoning refers to a
way of reasoning about the image that is based in symbolic logic.  As
noted above, this differs from most of what is called {\it visual
  reasoning} in the literature, where large language models (rather
than logic) are usually used to generate implications.  In the
taxonomy Daniel Kahnemahn laid out in {\it Thinking Fast and
  Slow}\cite{kahneman2011thinking}, the logical visual reasoning we
are interested in would be an instance of slow, System 2, reasoning,
whereas the usual literature is (arguably) concerned with System 1.
As such, it loses many of the desirable characteristics
(e.g. soundness and completeness) of logical reasoning.

Our final analysis will show that, as usually configured, CLIP {\bf
  cannot} serve as a basis for grounded logical reasoning about
images.  Our final contributions are thus:
\begin{enumerate}
\item We give a formal analysis of grounded logical reasoning in
  systems like CLIP; and
\item We show that such systems ultimately cannot support such reasoning.
\end{enumerate}

We by no means see this negative result as the end of the story.
Indeed, our formal analysis will point strongly to CLIP's use of
cosine similarity as imposing severe geometric contraints on what a
grounded logical reasoning system based on CLIP would look like.
This points to the desirability of analyzing whether other ways of
measuring similarity between a caption and an image might prove better
suited to building such a system.

\section{CLIP}


We preface our analysis with a summary of the operation of CLIP --
full details are in \cite{pmlr-v139-radford21a}.  Given a set of
images $\{\, I_1, \ldots, I_m \,\}$ and a set of potential captions
$\{\, C_1, \ldots, C_n \,\}$ as input, CLIP will process the two sets
as follows:
  \begin{enumerate}
  \item Each image is passed through a fixed visual feature detector
    (often a vision-transformer) and the detected features are then
    passed into a learned embedding network which maps each $I_i$ to a
    vector $\i_i$ in some finite dimensional vector space $\L = \R^d$.
    Thus CLIP implicitly defines a function $f$ which maps images to
    vectors in $\L$.
  \item Similarly, each caption $C_i$ is passed through a language model
    (e.g GPT2 \cite{radford2019language}) to obtain vector embedding of the caption,
    which is then also passed into a an embedding network to produce a
    vector $\c_i \in \L$.  Thus CLIP implicitly defines a function
    $g$ which maps captions to $\L$.
  \item For each image $I_i$, the probability distribution on captions
    describing $I_i$ is determined by taking a softmax of the cosine
    similiarities between each $\c_j$ and $\i_i$.  That is,

    \[
      p = \sm \left( 
      \begin{bmatrix}
        \frac{\i_i \cdot \c_1}{ \| i_i\| \| c_1 \|} \\
        \vdots  \\        
        \frac{\i_i \cdot \c_n}{ \| i_i\| \| c_n \|}
      \end{bmatrix} \right) 
  \]
  where $\sm(
  \begin{bmatrix}   x_1 & \ldots & x_n  \end{bmatrix}^T)$ is defined
  as $ \frac{1}{\sum_{i=1}^n e^{x_i}}  \begin{bmatrix}e^{x_1} & \ldots & e^{x_n} \end{bmatrix}^T$
                                                                    
  \smallskip

  \end{enumerate}

  We note that CLIP has been shown to be a very effective model for
  recognizing a wide variety of classes \cite{pmlr-v139-radford21a},
  and indeed its flexibility seems to offer tremendous potential for
  using CLIP as a basis for visual reasoning in the wild.   This
  promise largely underwrites the interest in this paper in studying
  CLIP as a basis for grounded logical reasoning.

\section{Related Work}
The genesis of this paper is a desire to combine the seemingly
grounded object recognition capabilities of CLIP with the power of
traditional logical reasoning.  This places the current work within
the body of literature studying visual reasoning using language model;
some prominent examples include Flamingo \cite{alayrac2022flamingo},
Visual Comet \cite{park2020visualcomet} and more
generally the large contemporary literature on visual question
answering and visual commonsense reasoning.  While our analysis takes
place in the context of CLIP, similar analyses can be done for other
vision-language systems.

Our particular emphasis here is on compositionality in such systems,
and in particular in our theoretical analysis we will focus on some
geometric considerations on the latent space of CLIP.  This ties the
current work to the literature on compositionality in symbolic systems
\cite{sep-compositionality}.  The geometric analysis finds some precedent in the
work of Peter G{\"a}rdenfors \cite{gardenfors2004conceptual}.

There is other work specifically dealing with compositionality in
language models and language-vision models.  For example,
\cite{helwe2021reasoning} shows that the BERT-like language models
often fail at logical reasoning, while \cite{subramanian2019analyzing}
raises the question of whether VQA models are truly compositional.

\section{Geometric Analysis}

The question in this paper is whether or not CLIP (or a variant
thereof) can support grounded logical reasoning.  As noted, the
explicit tying of the images and the captions which desribe them
provides some kind of grounding.  Indeed, we can view the behavior of
CLIP as providing a kind of logical interpretation function: for any
fixed image $i$ and description $d$, we can think of $d$ being true in
$i$ exactly when CLIP gives higher probability to $d$ than to $\neg
d$.   In classical logic, such an interpretation then extends to a
notion of truth in a model for arbitrarily complex combinations of
atomic descriptions.  Moreover, this notion of truth is exactly
reflected in the inference rules of classical logic -- this is precisely
the soundness and completeness of first-order (or propositional)
logic.  That is, starting with a consistent set of axioms and applying
inference rules not only yields logical consequences of the axioms,
but every such consequence can be attained this way.

We would like to examine whether the grounding obtained by viewing
CLIP as providing an interpretation function can be extended in a
similar vein.  To address this question, we first sharpen our
definitions and introduce some notation.  Suppose we are working with
some fixed set of image descriptions $\D$, and further that $\D$ is
the closure under some set of logical operations of some generating
set of atomic descriptions $\A$.  For example, $\A$ might be a set of
natural language descriptions of images as belonging to a set of basic
categories, containing strings like ``An image of a cat'' and ``An
image of a dog''.  Then closing under logical operators would add
descriptions like ``An image of a cat or a dog'' and ``Not an image of
a cat'' (we will formalize these notions below).  For any $d \in \D$,
let us say that $d$ is true of $i$ when $d$ accurately describes $i$,
and write this as $i \cmodels d$ \footnote{We recursively define
  $i \cmodels d$ for any $d \in \D$ in Definition \ref{defn:cmodels}
  below.}  .  We will write $i \cproves d$ when CLIP assigns higher
probability to $d$ than to $\neg d$ given the image set $\{\, i \,\}$
and the caption set $\{\, d, \neg d \,\}$.  In analogy with the
soundness and completeness of propositional calculus, our hope is that
$i \cmodels d$ exactly when $i \cproves d$.



\bigskip

\def\N{\mathcal{N}} We begin our discussion by defining our basic
syntax and semantics.  Throughout, we will subscript standard logical
notation with $_C$ (e.g. $\clor$) to indicate that we are working with
related but different notions.  Since CLIP uses natural language
descriptions, this is slightly more complicated than the siutation in
propostional calculus, where atomic propositions are combined in a
natural way.  Let $\N$ denote the set of all English language
descriptions.  We will fix functions
$\phi_\lor: \N \times \N \to \N, \phi_\land : \N \times \N \to \N$ and
$\phi_\neg: \N \to \N$ which determine how logical operators combine
descriptions.  For example, we might have
$\phi_\lor({\text{``An image of a cat", ``An image of a dog"}}) =
{\text{``An image of either a cat or a dog"}}$.  Given these functions
and given $d,e \in \A$, we will write $d \lor_C e$ for
$\phi_\lor(d,e)$, and similarly for $\land_C$ and $\neg_C$.  The
syntax and semantics of these operations are defined recursively in
the expected way.

\begin{defn} Let $\A$ be any set of descriptions, which we take to be
  atomic.  Then we define $\D_\A$, the set of {\it well-formed
    sentences generated by $\A$} as the set of strings of the form:
  \begin{itemize}
  \item $d$ for $d \in \A$; OR
  \item $\cneg(d)$ for well-formed $d \in \D$; OR
  \item $(d \clor e)$ for well-formed $d,e \in \D$; OR
  \item $(d \cland e)$ for well-formed $d,e \in \D$
  \end{itemize}
  When context makes them unnecessary, we will omit the parentheses in
  well-formed sentences (adopting the usual precedence conventions for
  logical operators).
\end{defn}

We now turn to the models under consideration; these are meant to have
the same basic structure as CLIP and use cosine similarity to measure
the closeness of embeddings.

\begin{defn} Let $\I$ denote any set of images and let $\D$ be as in
  convention \ref{conv:desc}.  Then a {\it CLIP-like} model is a pair
  of functions $(f,g)$ where for some Euclidean vector space $\L$, we
  have that $f: \I \to \L$ and $g: \D \to \L$.  That is, $f$ and $g$
  embed images and descriptions into a shared latent space.

  If $C$ is CLIP-like, then for $i \in \I, d \in \D$ we define
  $\alpha(i, d)$ to be the cosine similarity between the respective
  embeddings of $i$ and $d$.  
  \[
    \alpha_C(i, d) := \frac{ f(i) \cdot g(d)}{\| f(i)\|  \| g(d)\|}
  \]
\end{defn}

We will use this notion to recursively define the relation $i \cmodels d$.

\begin{defn}\label{defn:cmodels}   Let $C=(f,g)$ be CLIP-like.
  For any image  $i$ and well-formed $d,e \in \D_\A$:
  \begin{itemize}
  \item If $d \in \A$, then $i \cmodels d$ if and only if $C$ assigns
    higher probability to $d$ then to $\neg d$ when $C$ uses
    input image $\{\, i \,\}$ and caption set $\{\, d, \neg d \,\}$.
    That is, $i \cmodels d$ if and only if $\alpha_C(i,d) >
    \alpha_C(i, \neg_C d)$.  
  \item $i \cmodels \cneg d$ if and only if it is not the case that $i
    \cmodels d$.
  \item $i \cmodels d \clor e$ if and only if $i \cmodels d$ or $i
    \cmodels e$.
  \item $i \cmodels d \cland e$ if and only if $i \cmodels d$ and $i
    \cmodels e$
  \end{itemize}
\end{defn}

We want to introduce some modest constraints on the types of images
and their potential descriptions which we will consider.

\begin{defn}
  Let $\I$ be a set of images and let $\A$ be a set of strings which
  represents atomic descriptions of elements of $\I$, as above.  Let
  $C$ be CLIP-like.

  \begin{itemize}
  \item We say that {\it $\I$ is describable by $\A$} if for $i \in \I$,
    there is some $d \in \A$ such that $i \cmodels d$ and conversely
    for every $d \in A$ there is some $i \in \I$ such that $i \cmodels
    d$.  
  \item If $\I$ is describable by $\A$ we say that $(\I, \A)$ is {\it
      separable} if:
    \begin{itemize}
    \item For every $i \in \I$, there is some $e \in \A$ such that $i
      \cmodels \cneg e$.
    \item For every $e \in \A$, there is some $i \in \I$ such that $i
      \cmodels \cneg e$.
    \end{itemize}
    Separability prevents a situation where all descriptions are true
    of any one image or any one description is true of all images.
  \end{itemize}
\end{defn}

\begin{conv}\label{conv:desc}
For the remainder of this paper,  we fix a set of images $\I$ and a set of
strings $\A$ with $\I$ describable by $\A$ and $(\I, \A)$
separable. Further, we shall abbreviate $\D_\A$ simply as $\D$.

\end{conv}

In our analysis, the following observations will be used repeatedly

\begin{remark} Let $C$ be CLIP-like with latent space $\L$.  For any
  two points $a,b \in \L$, let
  $\chi(x,y) := \frac{a \cdot b}{\| a \| \| b \|}$ denote the cosine
  similartiy between $x$ and $y$
  \begin{enumerate}
  \item $\chi(a,b)$ is not defined for $a=0$ or $b=0$.  
  \item The cosine similarity of $a,b$ is maximal when the angle
    between $a,b$ is 0 and $\chi(a,b) = 1$; it is minimal when the
    angle is $\frac{\pi}{2}$ and $\chi(a,b) = -1$
  \item The set of all $c \in \L$ such $\chi(a,c) = 1$ forms a ray in
    $\L$ as does the set of all $c \in \L$ with $\chi(a,c) = -1$.
    Further, the latter ray is the reflection of the former ray over
    the origin.  We refer to the latter ray as the {\it anti-ray} of
    the former ray.  Since $\chi$ is not defined at the origin, any
    ray is disjoint from its anti-ray.
  \end{enumerate}
\end{remark}

We now define $C$-completeness  for a CLIP-like structure.

\begin{defn}
  If $C$ is CLIP-like, we say that
  \begin{itemize}
  \item $C$ {\it respects basic descriptions} if for every $i \in \I, d \in
    \A$, if $i \cmodels P$ then $\alpha_C(i, d) = 1$.  
  \item $C$ {\it respects negation} if for $i \in \I, d \in
    \A$, if $i \cmodels \cneg d$ then 
    $\alpha_C(i,d) = -1$
  \item $C$ {\it respects disjunction} if for $i \in \I, d,e \in
    \D$, if $i \cmodels d \clor e$, then    $\alpha_C(i, d \clor e) = 1$.
  \item $C$ {\it respects conjunction} if for $i \in \I, d,e \in
    \D$, if  $i \cmodels d \cland e$, then
    $\alpha_C(p, d \cland e) = 1$.
  \item $C$ is $C$-{\it complete} if it respects descriptions,
    negation, disjunction and conjunction.
  \end{itemize}
\end{defn}


A simple induction (on sentence length) shows that if $C$ is
$C$-complete, then for {\bf any} description $d \in \D$ and image
$i \in \I$, if $i \cmodels d$ then $\alpha_C(i, d) =1$ and, {\it a
  fortiori}, $i \cproves d$.  The idea here
is that the relation $i \cproves d$ represents some kind of notion of
$d$ being inferred by $C$.  In first-order logic, completeness means
that everything which can be inferred IS inferred.  
We would like something similar for $\cproves$ derivations
-- in particular the ability to derive all true descriptions of an
image.

What we will see is that this desire places extreme geometric
constraints on the embeddings $f(i), g(d)$.  In fact, assuming that
$C$-completeness holds, we will show that all descriptions which are
true of {\it any} image will necessarily live on a single ray, while
all descriptions which are false of any image will live on that ray's
anti-ray.  This will imply that all descriptions are on both the ray
and it's anti-ray, which is a contradiction since the two are
disjoint.

\begin{prop}\label{prop:main}
  If $C = (f,g)$ is CLIP-like then $C$ is not $C$-complete.
\end{prop}

We will prove this via a sequence of lemmas.

\begin{lem}\label{lem:1}
  If $C = (f,g)$ is $C$-complete, then there
  exists a ray $\R$ in $\L$ such that $ f(\I) \subseteq \R$ and
  $g(\D) \subseteq \R$.
\end{lem}

\begin{proof}
  Fix $i,j \in \I$ and $d,e \in \D$ such that $i \models d, j \models
  e$.  We will show that $f(i), f(j), g(d), g(e)$ are all on the same
  ray.  
  Since $C$ respects basic descriptions, we have
  \begin{equation}
    \label{eq:1}
    \alpha_C(i, d) = 1
  \end{equation}
  Since $i \cmodels d$, we have $i \cmodels d \clor e$ as well.  Since
  $C$ respects disjunctions, we have
  \begin{equation}
    \label{eq:2}
  \alpha_C(i, d \clor e) = 1
\end{equation}
Similarly, we have

\begin{align}
\label{eq:3}  \alpha_C(j, e) &= 1 \\   
\label{eq:4}  \alpha_C(j, d \lor e) &= 1  
\end{align}
By \eqref{eq:3} and \eqref{eq:4}, $f(j)$ is on the same ray as $g(d
\clor e)$, as is $g(e)$. By \eqref{eq:1} and \eqref{eq:2}, these are on
the same ray as $f(i)$ and $g(d)$.
\end{proof}

\def\a{\mathbf{A}}
\begin{lem}\label{lem:2}
  If $C = (f,g)$ respects descriptions, disjunctions and negations,
  then there exists a ray $\a$ in $\L$ such that for every
  $i \in \I, d \in \D$, if $i \models \neg d$ then $g(d) \in \a$.
  Moreover, $\a$ is the anti-ray of the ray $\R$ guaranteed by Lemmma
  \ref{lem:1}.
\end{lem}
\begin{proof}
By Lemma \ref{lem:1}, $f(i) \in \R$.  Since $C$ respects negations,
$\alpha_C(p, d) = -1$; thus $g(d) \in \a$.
\end{proof}

Fix $\R$ as in Lemma \ref{lem:1} and $\a$ as in Lemma \ref{lem:2}.
We note that Lemma \ref{lem:2} showed that if $d$ is false for {\it
  any} $i \in \I$, then $g(d)
\in \a$.   We also know from Lemma \ref{lem:1} that if $d$ is true
{\it any} $i \in I$, then $g(d) \in \R$.
\begin{cor}\label{cor:1}
  For $C = (f,g)$ as above and any $d \in \D$,
  \begin{enumerate}
  \item If there is any $i \in \I$ for which $i \cmodels \cneg d$, then
    $g(d) \in \a$.
  \item If there exist $p, q \in \I$ such that $p \cmodels d, q
    \cmodels \neg_C d$, then $g(d) \in \R \cap \a$
  \end{enumerate}
\end{cor}

Finally we prove our main result.

\begin{proof}[Proof of Proposition \ref{prop:main}]
  Suppose, by way of contradiction, that $C = (f, g)$ is CLIP-like and
  $C$-complete.  Fix $d \in \A$.  Since $\I$ is describable by $\A$,
  there exists some $i \in \I$ such that $i \models d$.  By Lemma
  \ref{lem:1}, $f(i) \in \R$ and $g(d) \in \R$.  Using separability
  and describability, choose $j \in \I$ and $e \in \A$ such that
  $j \models \neg d \land e$.  Since $j \models \neg d$, we have
  $g(d) \in \A$ as well by Corollary \ref{cor:1}.  However, we have
  that $\alpha_C(i, d) = 1$ so that $f(i) \in \a$ as well.  Thus
  $\{\, f(i), g(d) \,\} \subseteq \R \cap \a$. But
  $\R \cap \a = \emptyset$, contradicting that $f,g$ are maps into
  $\L$.
\end{proof}

\section{Conclusions and Future Work}

Our theoretical analysis shows no matter how accurately a CLIP-like
model detects basic categories, this cannot extend to arbitrary
boolean combinations.  This is by no means a fatal blow to the use of
CLIP in compostional reasoning; indeed two immediate possibilities
present themselves.  The most obvious is simply to use CLIP for basic
category recognition and deploy an external system to handle
composition of categories.  For example, to determine whether an image
is described by $P \land \neg (Q \lor R)$ we can simply run CLIP to
recognize each of the categories $P, Q, R$ and combine the results
using something like fuzzy logic.  In a sense, this is a validation of
using hybrid neurosymbolic models rather than trying to reason within
a vision-language model.

That said, there are reasons it would be desirable to work in a visual
reasoning system in which the latent space was organized in a
logically coherent manner as discussed in the paper.  While this
cannot be done for a CLIP-like system using cosine similarity, it is
still open whether a similar system organized according to Euclidean
distance or some other metric might work.   This will be explored in
future work.

We also note that the results in this paper represent a kind of
limiting case -- the question of logical coherence basically amounts
to asking what the geometry of a perfectly coherent system would look
like.   We have not addressed what the possibilities would be for a
system where the requirement of perfection was relaxed, for example by
having $i \cmodels d$ correspong to $\alpha_C(i, d) \geq 1 - \epsilon$
for some small $\epsilon > 0$.  This too will be explored in future
work.

\bibliography{paper}

\begin{thebibliography}{10}
\providecommand{\bibitemdeclare}[2]{}
\providecommand{\surnamestart}{}
\providecommand{\surnameend}{}
\providecommand{\urlprefix}{Available at }
\providecommand{\url}[1]{\texttt{#1}}
\providecommand{\href}[2]{\texttt{#2}}
\providecommand{\urlalt}[2]{\href{#1}{#2}}
\providecommand{\doi}[1]{doi:\urlalt{https://doi.org/#1}{#1}}
\providecommand{\eprint}[1]{arXiv:\urlalt{https://arxiv.org/abs/#1}{#1}}
\providecommand{\bibinfo}[2]{#2}

\bibitemdeclare{article}{alayrac2022flamingo}
\bibitem{alayrac2022flamingo}
\bibinfo{author}{Jean{-}Baptiste \surnamestart Alayrac\surnameend},
  \bibinfo{author}{Jeff \surnamestart Donahue\surnameend},
  \bibinfo{author}{Pauline \surnamestart Luc\surnameend},
  \bibinfo{author}{Antoine \surnamestart Miech\surnameend},
  \bibinfo{author}{Iain \surnamestart Barr\surnameend}, \bibinfo{author}{Yana
  \surnamestart Hasson\surnameend}, \bibinfo{author}{Karel \surnamestart
  Lenc\surnameend}, \bibinfo{author}{Arthur \surnamestart Mensch\surnameend},
  \bibinfo{author}{Katherine \surnamestart Millican\surnameend},
  \bibinfo{author}{Malcolm \surnamestart Reynolds\surnameend},
  \bibinfo{author}{Roman \surnamestart Ring\surnameend}, \bibinfo{author}{Eliza
  \surnamestart Rutherford\surnameend}, \bibinfo{author}{Serkan \surnamestart
  Cabi\surnameend}, \bibinfo{author}{Tengda \surnamestart Han\surnameend},
  \bibinfo{author}{Zhitao \surnamestart Gong\surnameend}, \bibinfo{author}{Sina
  \surnamestart Samangooei\surnameend}, \bibinfo{author}{Marianne \surnamestart
  Monteiro\surnameend}, \bibinfo{author}{Jacob~L. \surnamestart
  Menick\surnameend}, \bibinfo{author}{Sebastian \surnamestart
  Borgeaud\surnameend}, \bibinfo{author}{Andy \surnamestart Brock\surnameend},
  \bibinfo{author}{Aida \surnamestart Nematzadeh\surnameend},
  \bibinfo{author}{Sahand \surnamestart Sharifzadeh\surnameend},
  \bibinfo{author}{Mikolaj \surnamestart Binkowski\surnameend},
  \bibinfo{author}{Ricardo \surnamestart Barreira\surnameend},
  \bibinfo{author}{Oriol \surnamestart Vinyals\surnameend},
  \bibinfo{author}{Andrew \surnamestart Zisserman\surnameend} \&
  \bibinfo{author}{Kar{\'{e}}n \surnamestart Simonyan\surnameend}
  (\bibinfo{year}{2022}): \emph{\bibinfo{title}{Flamingo: a Visual Language
  Model for Few-Shot Learning}}.
\newblock
  \urlprefix\url{http://papers.nips.cc/paper\_files/paper/2022/hash/960a172bc7fbf0177ccccbb411a7d800-Abstract-Conference.html}.

\bibitemdeclare{inproceedings}{antol2015vqa}
\bibitem{antol2015vqa}
\bibinfo{author}{Stanislaw \surnamestart Antol\surnameend},
  \bibinfo{author}{Aishwarya \surnamestart Agrawal\surnameend},
  \bibinfo{author}{Jiasen \surnamestart Lu\surnameend},
  \bibinfo{author}{Margaret \surnamestart Mitchell\surnameend},
  \bibinfo{author}{Dhruv \surnamestart Batra\surnameend},
  \bibinfo{author}{C.~Lawrence \surnamestart Zitnick\surnameend} \&
  \bibinfo{author}{Devi \surnamestart Parikh\surnameend}
  (\bibinfo{year}{2015}): \emph{\bibinfo{title}{{VQA:} Visual Question
  Answering}}.
\newblock In: {\slshape \bibinfo{booktitle}{2015 {IEEE} International
  Conference on Computer Vision, {ICCV} 2015, Santiago, Chile, December 7-13,
  2015}}, \bibinfo{publisher}{{IEEE} Computer Society}, pp.
  \bibinfo{pages}{2425--2433}, \doi{10.1109/ICCV.2015.279}.

\bibitemdeclare{book}{gardenfors2004conceptual}
\bibitem{gardenfors2004conceptual}
\bibinfo{author}{Peter \surnamestart Gardenfors\surnameend}
  (\bibinfo{year}{2004}): \emph{\bibinfo{title}{Conceptual spaces: The geometry
  of thought}}.
\newblock \bibinfo{publisher}{MIT press}, \doi{10.7551/mitpress/2076.001.0001}.

\bibitemdeclare{inproceedings}{helwe2021reasoning}
\bibitem{helwe2021reasoning}
\bibinfo{author}{Chadi \surnamestart Helwe\surnameend},
  \bibinfo{author}{Chlo{\'{e}} \surnamestart Clavel\surnameend} \&
  \bibinfo{author}{Fabian~M. \surnamestart Suchanek\surnameend}
  (\bibinfo{year}{2021}): \emph{\bibinfo{title}{Reasoning with
  Transformer-based Models: Deep Learning, but Shallow Reasoning}}.
\newblock In \bibinfo{editor}{Danqi \surnamestart Chen\surnameend},
  \bibinfo{editor}{Jonathan \surnamestart Berant\surnameend},
  \bibinfo{editor}{Andrew \surnamestart McCallum\surnameend} \&
  \bibinfo{editor}{Sameer \surnamestart Singh\surnameend}, editors: {\slshape
  \bibinfo{booktitle}{3rd Conference on Automated Knowledge Base Construction,
  {AKBC} 2021, Virtual, October 4-8, 2021}}, \doi{10.24432/C5W300}.

\bibitemdeclare{book}{kahneman2011thinking}
\bibitem{kahneman2011thinking}
\bibinfo{author}{Daniel \surnamestart Kahneman\surnameend}
  (\bibinfo{year}{2011}): \emph{\bibinfo{title}{Thinking, fast and slow}}.
\newblock \bibinfo{publisher}{Macmillan}.

\bibitemdeclare{inproceedings}{park2020visualcomet}
\bibitem{park2020visualcomet}
\bibinfo{author}{Jae~Sung \surnamestart Park\surnameend},
  \bibinfo{author}{Chandra \surnamestart Bhagavatula\surnameend},
  \bibinfo{author}{Roozbeh \surnamestart Mottaghi\surnameend},
  \bibinfo{author}{Ali \surnamestart Farhadi\surnameend} \&
  \bibinfo{author}{Yejin \surnamestart Choi\surnameend} (\bibinfo{year}{2020}):
  \emph{\bibinfo{title}{VisualCOMET: Reasoning About the Dynamic Context of a
  Still Image}}.
\newblock In \bibinfo{editor}{Andrea \surnamestart Vedaldi\surnameend},
  \bibinfo{editor}{Horst \surnamestart Bischof\surnameend},
  \bibinfo{editor}{Thomas \surnamestart Brox\surnameend} \&
  \bibinfo{editor}{Jan{-}Michael \surnamestart Frahm\surnameend}, editors:
  {\slshape \bibinfo{booktitle}{Computer Vision - {ECCV} 2020 - 16th European
  Conference, Glasgow, UK, August 23-28, 2020, Proceedings, Part {V}}},
  {\slshape \bibinfo{series}{Lecture Notes in Computer Science}}
  \bibinfo{volume}{12350}, \bibinfo{publisher}{Springer}, pp.
  \bibinfo{pages}{508--524}, \doi{10.1007/978-3-030-58558-7\_30}.

\bibitemdeclare{inproceedings}{pmlr-v139-radford21a}
\bibitem{pmlr-v139-radford21a}
\bibinfo{author}{Alec \surnamestart Radford\surnameend},
  \bibinfo{author}{Jong~Wook \surnamestart Kim\surnameend},
  \bibinfo{author}{Chris \surnamestart Hallacy\surnameend},
  \bibinfo{author}{Aditya \surnamestart Ramesh\surnameend},
  \bibinfo{author}{Gabriel \surnamestart Goh\surnameend},
  \bibinfo{author}{Sandhini \surnamestart Agarwal\surnameend},
  \bibinfo{author}{Girish \surnamestart Sastry\surnameend},
  \bibinfo{author}{Amanda \surnamestart Askell\surnameend},
  \bibinfo{author}{Pamela \surnamestart Mishkin\surnameend},
  \bibinfo{author}{Jack \surnamestart Clark\surnameend},
  \bibinfo{author}{Gretchen \surnamestart Krueger\surnameend} \&
  \bibinfo{author}{Ilya \surnamestart Sutskever\surnameend}
  (\bibinfo{year}{2021}): \emph{\bibinfo{title}{Learning Transferable Visual
  Models From Natural Language Supervision}}.
\newblock In \bibinfo{editor}{Marina \surnamestart Meila\surnameend} \&
  \bibinfo{editor}{Tong \surnamestart Zhang\surnameend}, editors: {\slshape
  \bibinfo{booktitle}{Proceedings of the 38th International Conference on
  Machine Learning, {ICML} 2021, 18-24 July 2021, Virtual Event}}, {\slshape
  \bibinfo{series}{Proceedings of Machine Learning Research}}
  \bibinfo{volume}{139}, \bibinfo{publisher}{{PMLR}}, pp.
  \bibinfo{pages}{8748--8763}.
\newblock \urlprefix\url{http://proceedings.mlr.press/v139/radford21a.html}.

\bibitemdeclare{article}{radford2019language}
\bibitem{radford2019language}
\bibinfo{author}{Alec \surnamestart Radford\surnameend},
  \bibinfo{author}{Jeffrey \surnamestart Wu\surnameend}, \bibinfo{author}{Rewon
  \surnamestart Child\surnameend}, \bibinfo{author}{David \surnamestart
  Luan\surnameend}, \bibinfo{author}{Dario \surnamestart Amodei\surnameend},
  \bibinfo{author}{Ilya \surnamestart Sutskever\surnameend} et~al.
  (\bibinfo{year}{2019}): \emph{\bibinfo{title}{Language models are
  unsupervised multitask learners}}.
\newblock {\slshape \bibinfo{journal}{OpenAI blog}}
  \bibinfo{volume}{1}(\bibinfo{number}{8}), p.~\bibinfo{pages}{9}.

\bibitemdeclare{article}{ramesh2022hierarchical}
\bibitem{ramesh2022hierarchical}
\bibinfo{author}{Aditya \surnamestart Ramesh\surnameend},
  \bibinfo{author}{Prafulla \surnamestart Dhariwal\surnameend},
  \bibinfo{author}{Alex \surnamestart Nichol\surnameend},
  \bibinfo{author}{Casey \surnamestart Chu\surnameend} \& \bibinfo{author}{Mark
  \surnamestart Chen\surnameend} (\bibinfo{year}{2022}):
  \emph{\bibinfo{title}{Hierarchical Text-Conditional Image Generation with
  {CLIP} Latents}}.
\newblock {\slshape \bibinfo{journal}{CoRR}} \bibinfo{volume}{abs/2204.06125},
  \doi{10.48550/arXiv.2204.06125}.
\newblock \eprint{2204.06125}.

\bibitemdeclare{article}{rombach2022highresolution}
\bibitem{rombach2022highresolution}
\bibinfo{author}{Robin \surnamestart Rombach\surnameend},
  \bibinfo{author}{Andreas \surnamestart Blattmann\surnameend},
  \bibinfo{author}{Dominik \surnamestart Lorenz\surnameend},
  \bibinfo{author}{Patrick \surnamestart Esser\surnameend} \&
  \bibinfo{author}{Bj{\"{o}}rn \surnamestart Ommer\surnameend}
  (\bibinfo{year}{2022}): \emph{\bibinfo{title}{High-Resolution Image Synthesis
  with Latent Diffusion Models}}, pp. \bibinfo{pages}{10674--10685}.
\newblock \doi{10.1109/CVPR52688.2022.01042}.

\bibitemdeclare{article}{subramanian2019analyzing}
\bibitem{subramanian2019analyzing}
\bibinfo{author}{Sanjay \surnamestart Subramanian\surnameend},
  \bibinfo{author}{Sameer \surnamestart Singh\surnameend} \&
  \bibinfo{author}{Matt \surnamestart Gardner\surnameend}
  (\bibinfo{year}{2019}): \emph{\bibinfo{title}{Analyzing Compositionality in
  Visual Question Answering}}.
\newblock \urlprefix\url{https://vigilworkshop.github.io/static/papers/43.pdf}.

\bibitemdeclare{incollection}{sep-compositionality}
\bibitem{sep-compositionality}
\bibinfo{author}{Zoltán~Gendler \surnamestart Szabó\surnameend}
  (\bibinfo{year}{2022}): \emph{\bibinfo{title}{{Compositionality}}}.
\newblock In \bibinfo{editor}{Edward~N. \surnamestart Zalta\surnameend} \&
  \bibinfo{editor}{Uri \surnamestart Nodelman\surnameend}, editors: {\slshape
  \bibinfo{booktitle}{The {Stanford} Encyclopedia of Philosophy}},
  \bibinfo{edition}{{F}all 2022} edition, \bibinfo{publisher}{Metaphysics
  Research Lab, Stanford University}.

\bibitemdeclare{article}{vaswani2017attention}
\bibitem{vaswani2017attention}
\bibinfo{author}{Ashish \surnamestart Vaswani\surnameend},
  \bibinfo{author}{Noam \surnamestart Shazeer\surnameend},
  \bibinfo{author}{Niki \surnamestart Parmar\surnameend},
  \bibinfo{author}{Jakob \surnamestart Uszkoreit\surnameend},
  \bibinfo{author}{Llion \surnamestart Jones\surnameend},
  \bibinfo{author}{Aidan~N. \surnamestart Gomez\surnameend},
  \bibinfo{author}{Lukasz \surnamestart Kaiser\surnameend} \&
  \bibinfo{author}{Illia \surnamestart Polosukhin\surnameend}
  (\bibinfo{year}{2017}): \emph{\bibinfo{title}{Attention is All you Need}},
  pp. \bibinfo{pages}{5998--6008}.
\newblock
  \urlprefix\url{https://proceedings.neurips.cc/paper/2017/hash/3f5ee243547dee91fbd053c1c4a845aa-Abstract.html}.

\end{thebibliography}

\end{document}